%% file: paper.tex
\documentclass{article}

\usepackage{mathtools}
\usepackage{amssymb}

\PassOptionsToPackage{numbers, compress}{natbib}

\usepackage[preprint]{neurips_2019}

\usepackage[utf8]{inputenc} 
\usepackage[T1]{fontenc}    
\usepackage{hyperref}       
\usepackage{url}            
\usepackage{booktabs}       
\usepackage{amsfonts}       
\usepackage{nicefrac}       
\usepackage{microtype}      

\usepackage{graphicx}
\usepackage[dvipsnames]{xcolor}

\setcitestyle{square}
\usepackage{amsmath}
\usepackage{amssymb}
\newcommand{\RR}{I\!\!R} 

\newcommand{\dec}{D}
\newcommand{\tran}{\mathcal{T}}

\newcommand{\featspace}{\mathcal{X}}
\newcommand{\statespace}{\mathcal{S}}
\newcommand{\acspace}{\mathcal{A}}
\newcommand{\rew}{\mathcal{R}}
\newcommand{\del}{\nabla}

\usepackage{amsthm}

\newcommand{\E}{\mathop{\mathbb{E}}}

\newcommand{\subheading}[1]{\subsection{#1}}

\DeclareMathOperator*{\argmax}{arg\,max}

\setcitestyle{square}

\title{Extracting Incentives From Black-Box Decisions}

\author{%
  Yonadav Shavit\thanks{These authors contributed equally.} \\
  Harvard SEAS\\
  \texttt{yonadav@g.harvard.edu} \\
   \And
   William S. Moses\footnotemark[1] \\
   MIT \\
   \texttt{wmoses@mit.edu} \\
}

\begin{document}

\maketitle

\begin{abstract}
    An algorithmic decision-maker incentivizes people to act in certain ways to receive better decisions.
    These incentives can dramatically influence subjects' behaviors and lives, and it is important that both decision-makers and decision-recipients have clarity on which actions are incentivized by the chosen model.
    While for linear functions, the changes a subject is incentivized to make may be clear, we prove that for many non-linear functions (e.g. neural networks, random forests), classical methods for interpreting the behavior of models (e.g. input gradients) provide poor advice to individuals on which actions they should take.
    In this work, we propose a mathematical framework for understanding algorithmic incentives as the challenge of solving
    a Markov Decision Process, where the state includes the set of input features, and the reward is a function of the model's output.
    We can then leverage the many toolkits for solving MDPs (e.g. tree-based planning, reinforcement learning) to identify the optimal actions each individual is incentivized to take to improve their decision under a given model.
    We demonstrate the utility of our method by estimating the maximally-incentivized actions in two real-world settings: 
    a recidivism risk predictor we train using ProPublica's COMPAS dataset, and
    an online credit scoring tool published by the Fair Isaac Corporation (FICO).
\end{abstract}
\input{1_Introduction.tex}
\input{2_Related_Work.tex}
\input{3_Framework.tex}

\input{4_Linear_Drawbacks.tex}
\input{5_Experiments.tex}
\input{6_Discussion.tex}
\input{7_Ack.tex}

\vspace{1cm}
\newpage

\bibliography{bibliography.bib}
\vspace{1cm}
\newpage
\appendix
\input{A_Proofs.tex}
\end{document}

%% file: 1_Introduction.tex
\section{Algorithm-Dictated Incentives}
When decision-makers determine our access to things we want (loans, college admissions, etc.), we inevitably wonder what we can do to make them more likely to decide in our favor. In an increasingly data-driven world, it is common for decision-makers to train machine-learning models to automate the decision-making process. These machine-learning-based decision methods are generally designed to maximize the accuracy of their predictions, without considering how the decision process might affect subjects' behavior.
Yet, such models do implicitly define incentives to those subjected to their decisions, with serious social consequences. 

Consider a variant of an example raised by \citet{eubanks_2018}: families in distress are both more likely to utilize public welfare programs (such as food aid and public housing) and separately more likely to default on their debt. A model trained to predict creditworthiness would, in turn, learn a correlation between utilizing aid and failure to repay. Thus, struggling parents may need to avoid using vital aid for fear they would be denied a loan in their time of need, further exacerbating a family's deprivation.

In this work, we propose a framework for studying the incentives defined by black-box decision functions, including nonlinear functions like neural networks and random forests, given only query access to the function. To the best of our knowledge, our work is the first to propose a definition for the incentives induced by arbitrary black-box decision functions, and to propose a method for determining the actions these decisions incentivize.

First, let us clarify terms. There exists a \textit{decision-maker}, who will at some point in the future make a decision about an individual, referred to as a \textit{decision subject}, based on that individual's \textit{state}. We say that this subject has \textit{agency} if they are capable of taking \textit{actions} to change their state, and thus affect the decision they will ultimately receive. An \textit{advice policy} is a recommendation for which actions to take at each state. By encouraging some actions over others, a decision-maker provides \textit{incentives} to individuals. The \textit{maximally-incentivized action} will, if executed, result in a better final decision than any alternative action.

While identifying an incentivized action from a given decision rule may seem straightforward, it can require surprising nuance. Consider the bimodal decision function in Figure \ref{fig:1di}. Whether the individual is "incentivised" to move left or right depends not only on the decision rule itself, but also how many resources (e.g. time, money) the individual has to expend on future actions.
\begin{figure}
    \centering
\minipage{0.38\textwidth}
\includegraphics[width=\linewidth]{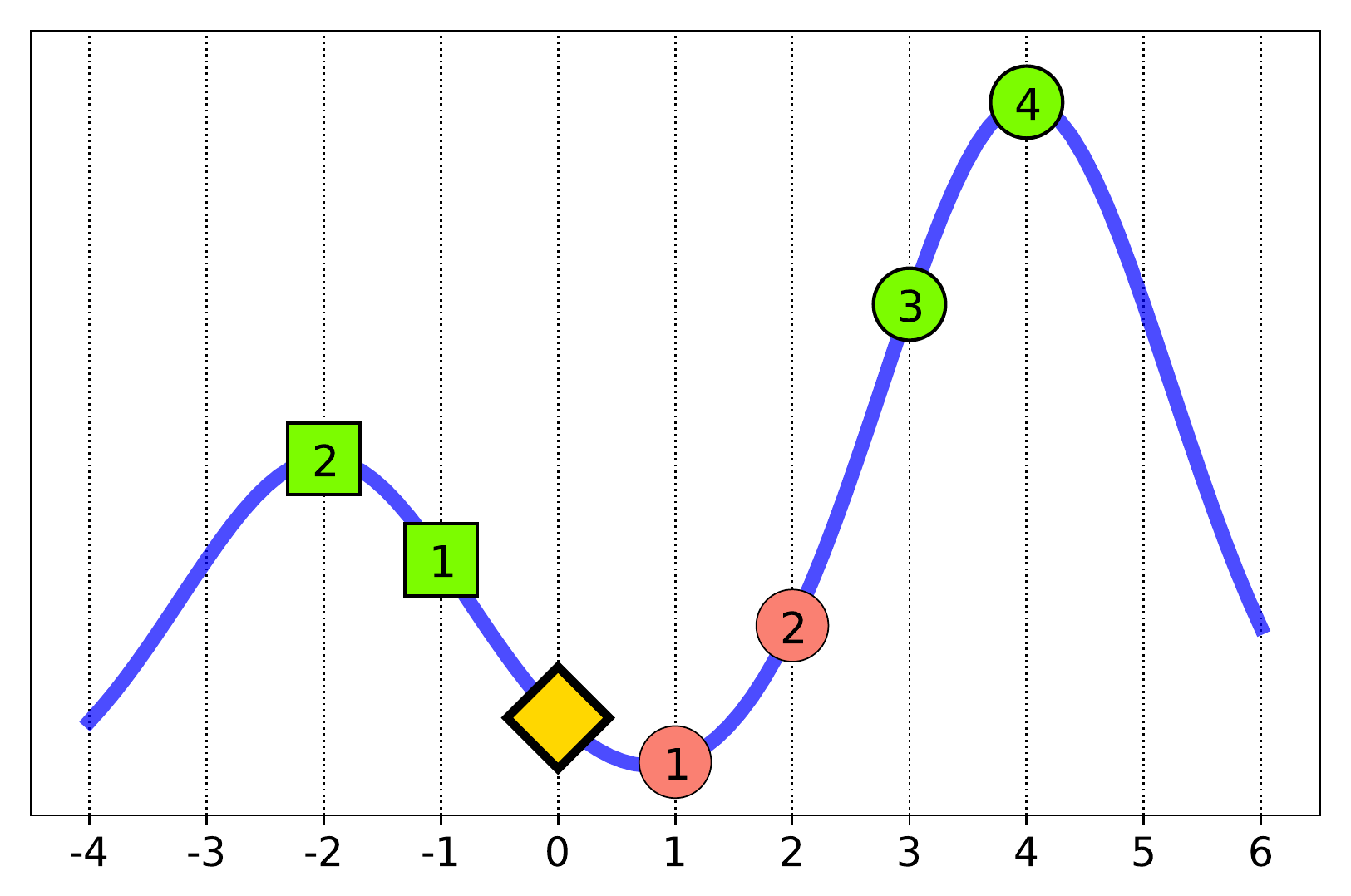}
\endminipage\hfill
\minipage{0.58\textwidth}
    \caption{A subject (gold diamond) wants to maximize the decision received from a decision rule (blue line) by stepping right or left, or staying in place. Each step expends 1 resource (e.g. time, money). The decision is only made when all the resources are used up, or the individual has stopped. If the individual has two resources or fewer, they should head towards the left peak (squares). On the other hand, if they have 3 resources or more, they should move right, because even though their interim state may seem worse, their final decision will be better (circles).}
\endminipage

    \label{fig:1di}
\end{figure}

\subheading{Our Contributions}

As our primary contribution, in Section \ref{sec:framework} we propose a framework for both evaluating the quality of an advice policy on a decision model, and for generating high-quality advice policies from decision models. Our method only requires black-box query access to the models, and can thus be utilized both by decision-makers wishing to evaluate the incentives of their models, or by auditors looking for incentivized behaviors but posessing only API-access to the true decision-making pipeline.

The primary insight of this work is that a subject's agency, their process of changing to improve the outcome they receive from a decision-maker, can be understood as a Markov decision process (MDP). In this "agency MDP", the state includes the input features to the decision, the actions represent the choices available to the individual, and the reward is the decision when the MDP terminates and 0 otherwise. This also allows us to model the case where a decision-subject can take a sequence of actions, and the precise sequencing of actions those actions matters. For example, "studying for an exam" may have a different effect before vs. after "getting a full night's sleep". We can compute a decision model's maximally-incentivized action as the current action suggested by an optimal policy of this agency MDP.

In Section \ref{sec:linear}, we prove that local linear approximations \cite{lundberg2016unexpected}, a popular alternative interpretability method, fail to recover the maximally-incentivized actions for a wide class of decision functions.

Finally, in Section \ref{sec:experiments}, we demonstrate the utility of our framework in computing maximally-incentivized actions in two real-world decision settings: a recidivism predictor based on ProPublica's COMPAS dataset \cite{angwin_larson_kirchner_mattu_2019}, and an online credit-scoring tool published by FICO \cite{myfico}. 
We approximate these decision models' maximally-incentivized actions by estimating the optimal policy for their agency-MDPs using MCTS \cite{browne2012survey}. 
We demonstrate that our approach outperforms several baseline advice policies, and investigate the incentives provided by these decision models, including surprising findings about the impact of model choice on agency across demographic groups.

\subheading{When Should We Study Incentives?}

First and foremost, in some sectors, as in credit scoring \cite{usc_2016}, the law itself mandates that the public has a right to know how to improve their received decisions. US law requires credit raters to provide "adverse action notices" 
which must include up to $4$ "key factors that adversely affected the credit score of the consumer"
\cite{usc_2016}. Arguably, key factors are those factors that would yield the greatest improvement (in the short or long term) to an individual's score, i.e. the best advice to an individual. 

Even when not legally mandated, there is inherent value, to both an individual/customer and to society, in understanding the incentives dictated by decision systems.
Identifying incentivized actions can be understood as a type of algorithmic interpretability \cite{doshi2017towards, lipton2016mythos}, where the objective is to shed light on the effect that the choice of model will have on individual behavior in the decision-receiving population.
Separately, analyzing incentives can help determine whether a particular model provides its subjects sufficient agency - or if it provides agency unequally across different groups \cite{milli2019social, hu2019disparate}.

While decision recipients clearly value the agency that comes from knowing their incentivized actions, one might reasonably ask whether an algorithm owner would ever want to reveal to subjects the best way to game the decision function without being required. After all, \citet{hardt2016strategic} rightfully consider that agents following their incentives may shift the distribution of subjects such that the same features no longer predict the same outcome, reducing the decision model's accuracy. In this case, a model owner might reasonably want to retrain their model -- potentially rendering pointless any incentive-motivated actions. 
There are several cases where exposing incentivized actions remains valuable. First, if the decision model is causal or has already been trained on incentive-following individuals, it may be robust to individual gaming. It is also the case that individuals may glean incentives by themselves -- whether through leaks, or by observing numerous feature/decision pairs. This means that subjects will discover and act on the model's implicit incentives, and bear the corresponding social costs, whether or not the practitioner has considered them.

%% file: 2_Related_Work.tex
\section{Related Work}
\label{sec:relatedwork}
Computer science has only recently begun to reckon with the agency of individuals subject to complex automated decision rules, and the incentives they experience. One can understand the problem of identifying incentives as a sub-problem of algorithmic interpretability \cite{guidotti2019survey, doshi2017towards}, though interpretability is a loaded term \cite{lipton2016mythos}, and as we will show in Section \ref{sec:linear} many methods from the literature fail to accurately recover an algorithm's incentives.

Incentives have been studied for a long time in the economics literature, particularly around the misalignment of incentives between different parties \cite{kerr1975folly, grossman1992analysis}. Modeling individuals as subject to MDPs has long been practice in economics \cite{rust1994structural}. For example, \citet{keane1997career} construct a $5$-action MDP to study the career choices of young men, validating their model on real-world data. However, in this literature, agents are subject to reward rules far simpler than today's algorithmic decision functions.

\citet{hardt2016strategic} consider agency through the lens of individuals ``gaming'' the classifier, thus reducing its accuracy. They explore means for reducing the ability of agents to game a classifier, though doing so may unfairly reduce the agency of previously-disadvantaged groups (\citet{milli2019social, hu2019disparate}). Recent work has suggested that this ``effort-unfairness'' across groups may appear even in non-strategic classifiers \cite{heidari2019long}.
On the opposite end, \citet{kleinberg2018classifiers} study the problem of designing a decision rule solely for the purpose of imposing a particular set of incentives on an individual, given a linear model of that individual's available actions.

Recent work has begun to wrestle with the question of extracting incentives from existing decision models that were not trained to consider the behavior of their subjects. \citet{wachter2017counterfactual} propose `explaining' incentives to individuals by identifying (e.g. via MILPs \cite{russell2019efficient} or SMT-solvers \cite{karimi2019model}) a `similar' hypothetical individual who would've received a better decision. This line of work implicitly assumes that an informative counterfactual is ``close'' in feature space to the original, but does not ask what sequence of actions would allow the individual to attain the counterfactual -- or whether such a sequence even exists. \citet{ustun2019actionable} first proposed the question most directly related to ours, of finding actionable ``recourse'' for individuals subject to the decisions of a linear model. Their approach is applicable to categorical/linear transition models and white-box linear decision functions (including linear and logistic regressions), and provides optimal incentives within these settings.

Two independent concurrent works have explored the challenge of identifying multi-step action sequences for differentiable decision functions like neural networks. Much like this work, \citet{Ramakrishnan2019synthesizing} model the problem as choosing a sequence of actions to maximize a differentiable decision function, by exhaustively searching the space of acceptable discrete action sequences, and for each sequence adjusting each action by locally solving a differentiable optimization problem. \citet{Joshi2019towards} learn a latent embedding of the data manifold, and then assume that an individual acts by perturbing their features in this latent space (thus ensuring that the result of each ``step'' yields a feature vector near the original data distribution). They then find an action sequence by taking a series of gradient steps in latent space to approximately maximize the combined embedding and decision models.

To our knowledge, ours is the first work to propose a means of identifying incentivized action sequences for arbitrary black-box (possibly non-differentiable) models, and to have identified the connection between extracting recourse and reinforcement learning.

%% file: 3_Framework.tex
\section{Framework}
\label{sec:framework}
At a high level, we propose that to properly understand how an individual is incentivized to act, we must first define the actions available to an individual, and their effects. Then, the individual is incentivized to take whichever action will modify their current state such that, after executing a sequence of additional actions, they will reach a final state that maximizes their received decision.

Consider an individual with state $s$ drawn from a universe of individuals $\statespace$. Let $I_x(s) = x$ return the decision features $x \in \featspace$ of the individual, and let $\dec: \featspace \rightarrow \RR^+$ be a decision function that maps the individual's features to a positive real-valued decision.

While there may be many different ways that an individual may wish to change their received decision, we restrict our attention to the simple case where the individual wants to maximize the value of the decision $\dec(x)$ they receive. We can enumerate a set of actions $a \in \acspace$ that the individual can execute in order to change their state $s$, which may consequently affect their received decision.
We specify a transition model $\mathcal{T}: \statespace \times \acspace \rightarrow \Delta(\statespace)$ that describes the individual's new state $s' \sim \tran(s, a)$ as a draw from the distribution of possible consequences of taking action $a$ at state $s$. One of the benefits of this framing is that the actions can be defined as the intuitive choices an individual can make, and each action can have complex, state-dependent effects. 

As an example, let's consider the case of an applicant seeking to maximize their credit score before applying for a home loan: $s$ would be the overall state of the individual, including how much time they had left before they wanted to buy a home; $x$ would be the subset of their features used in a credit model; $\dec$ would be the credit scoring function; $a$ could be the action of spending a month paying down the debt on a credit card; $\tran(s, a)$ would be the impact of that payment on the individual's debt-to-income ratio and credit utilization, along with the decrease in time remaining before the individual plans to apply for the loan.

We define an ``agency MDP'' as the Markov decision process where the state is the individual's overall state $s$, the actions available are $a \in \acspace$, and the transition model is $\tran(s, a)$. We specify a terminal function $\text{end}(s)$ that determines whether the sequence has ended, and define the reward function $\rew$:
\begin{equation}
    \rew(s) = 
    \begin{cases}
    \dec(I_x(s)) & \text{if}~ \text{end}(s) \\
    0 & \text{otherwise}
    \end{cases}
\end{equation}

Given a deterministic advice policy $\pi$ from the set of policies $\Pi: \statespace \rightarrow \acspace$ where $\pi(s)$ maps state $s$ to advised action $a$, let $H_\pi(s)$ be the distribution of end-states resulting from "rolling out" $\pi$ starting at state $s$ via the following procedure: starting at $s_0=s$, take action $a_0 = \pi(s_0)$ to get a new state $s_1 \sim \tran(s_0, a)$, then execute a new action $a_1=\pi(s_1)$ to get a new state $s_2$, and repeat, until arriving at a terminal state $s_f$ such that $\text{end}(s_f)=\text{True}$. We often take $\text{end}(s)$ as asserting whether the available \text{resources} (time, money, etc) necessary to act have run out. Attributing such costs to actions is important: we want to evaluate the utility of an action while enforcing realistic constraints on the decision-recipient (e.g. a student preparing for an exam has only a fixed amount of time to study).

We say that an individual with state $s$ is $\textit{incentivized}$ to execute action $a^*$ if that action will \textbf{maximally} improve their eventual expected decision, more than any alternative action. More precisely:
\begin{equation}
\label{eq:incentivequality}
    a^* = \argmax_{a \in A} \left ( \max_{\pi \in \Pi} \E_{s_f \sim H_\pi(s)} \left [\dec\left(I_x(s_f)\right)\right ] \right ) = \pi^*(s)
\end{equation}
where $\pi^*$ is a policy maximizing the non-discounted expected reward on the agency MDP. We can approximate this optimal advice policy by leveraging the countless approaches in the reinforcement learning and planning literature dedicated to the task of computing optimal policies for MDPs \cite{kaelbling1996reinforcement, kolobov2012planning}.
We can also compare the effectiveness of different advice policies by comparing the expected decision resulting from following each advice policy: $\E_{s_f \sim H_\pi(s)} \left [\dec\left(I_x(s_f)\right)\right ]$.

Note that we have defined the maximally-incentivized action as the action maximizing the resulting decision, conditioned on the fact that the subject will, after following the incentive, continue to follow an optimal action policy (i.e. not changing their mind). Alternative models of individual behavior can largely be incorporated into this same framework by modifying the state and transition model - for example, by letting $s$ contain "commitment to the policy", having certain actions decrease that quantity, and reverting to a fixed alternative policy if commitment ever drops below a threshold.

A major principle of our framework is that one cannot determine which behaviors are incentivized by a decision-model without explicitly considering the actions available to the decision-receiver. For example, advising an individual to "open more credit cards" to improve their credit score may backfire and reduce their score if having more cards increases their likelihood of taking on unmanageable debt. Choosing the right transition model can be tricky. While there does exist a true transition model (reality), any approximate model chosen will significantly impact the incentives any method will discover. That said, as previously mentioned, transition models validated on real-world data have been used extensively in the economics literature \cite{rust1994structural}. We illustrate the impact of different transition models on the incentives generated in a real credit scoring model in Section \ref{sec:experiments}.

%% file: 4_Linear_Drawbacks.tex
\section{Problems with Local Approximations as Advice Policies}
\label{sec:linear}

One might reasonably hope that, instead of needing to bother with solving an agency MDP, a decision model's local structure can serve as a good heuristic for the incentives dictated by that model. After all, local approximation methods \cite{ribeiro2016should, baehrens2010explain} have gained widespread adoption as one of the most effective methods for interpreting otherwise-inscrutable machine learning models \cite{lundberg2016unexpected}. 

Unfortunately, even for simple non-linear decision functions, local-approximation-based advice can be dangerously wrong. We have already seen the example in Figure \ref{fig:1di}, in which a locally-improving policy would trap the individual at a local maximum and never achieve the better outcome that was available to them. Following local advice may result in individuals receiving sub-optimal decisions even when the decision function is monotonic: in Figure \ref{fig:2di} we can see that a locally-optimizing policy would lead a subject to a substantially inferior decision over the course of $6$ steps, even when each locally-optimal action seemed to improve the decision.

Below, we'll lay out a definition for locally-optimal advice, and prove that it only recovers the maximally-incentivized actions for a narrow class of decision functions.

\begin{figure}
    \centering
\minipage{0.3\textwidth}
    \includegraphics[width=\linewidth]{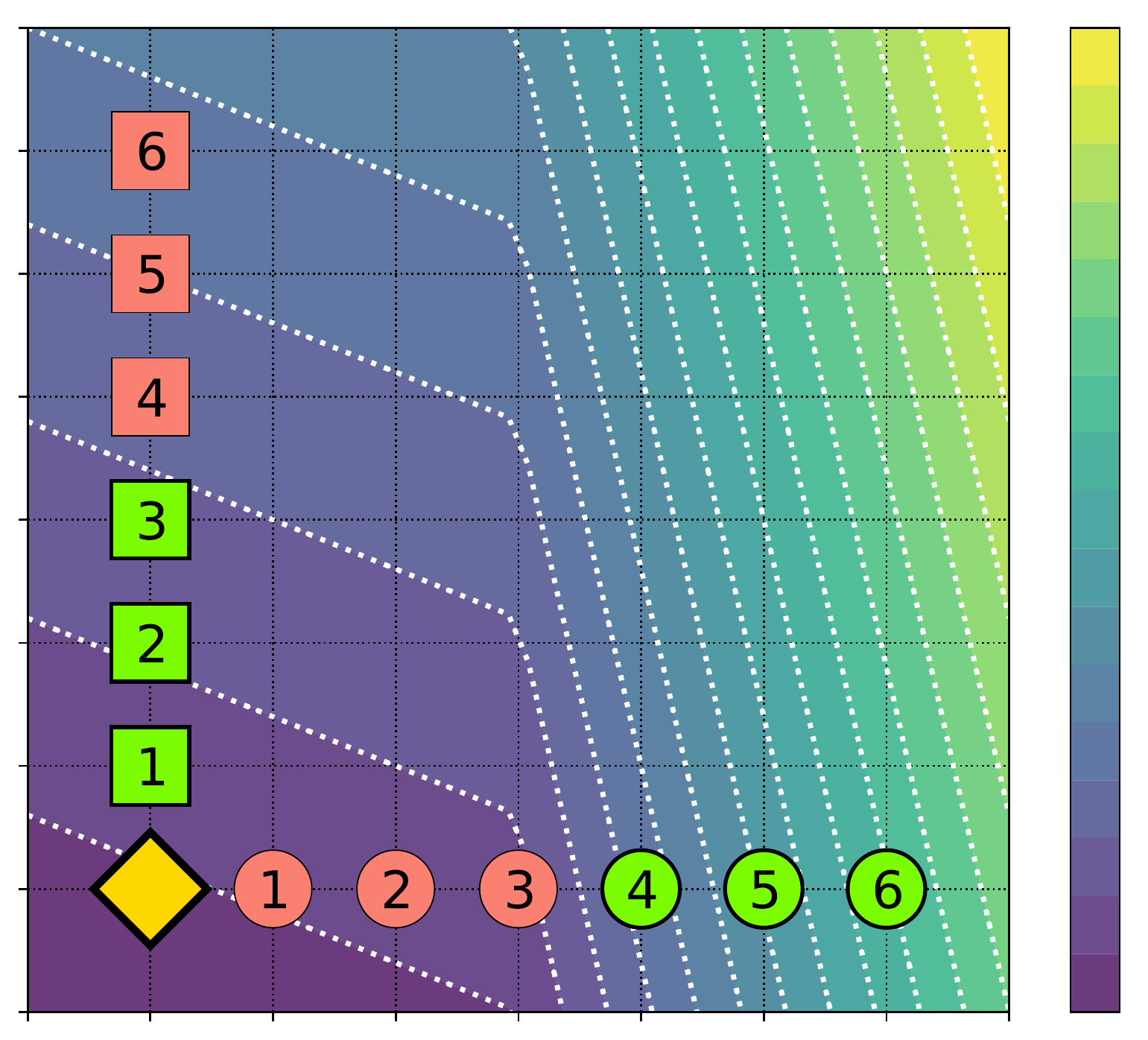}
\endminipage\hfill
\minipage{0.68\textwidth}
    \caption{A 2D monotonic decision function, with the lowest output in bottom-left (blue), and highest output in top-right (yellow). A subject in the bottom-left corner (gold diamond) can move $1$ grid unit each step. Circles represent the optimal agency-MDP policy given $4$ or more steps, while squares represent the greedy policy.}
\endminipage
    \label{fig:2di}
\end{figure}

\subheading{When local approximations work, and when they fail}

In our setting, we define a locally-optimal \textit{greedy} policy as a policy that chooses actions based only on maximizing the immediate improvement in the received decision:
\begin{equation}
\label{eq:greedy}
\pi_{local}(s) = \argmax_{a \in \acspace} \E_{s' \sim \tran(s,a)} \left [\dec\left(I_x(s')\right)\right ]
\end{equation}
When the function is differentiable, we can also define the \textit{gradient-following} policy for a state $s = [x, r]$ composed of only the decision features $x$ and the remaining available displacement $r$. Note that in this hypothetical transition model, we can manipulate every feature axis in $x$ independently.
\begin{equation}
\label{eq:gradient}
\pi_{gradient}(s) = \argmax_{a \in\{a' \in \RR^{dim(x)}~| ~\|a'\|_2 = \epsilon, ~\epsilon \ll 1\} } \dec(x + a) ~~= ~~ \epsilon \nabla_x\dec\left(x\right)
\end{equation}

The following are a series of theorems about this setting; the proofs are included in Appendix \ref{sec:proofs}.
\newtheorem{theorem}{Theorem}
\begin{theorem}
\label{gradIsOptimal}
If there exists a policy that is optimal independent of the number of resources remaining, it must be equivalent to the greedy policy.
\end{theorem}

\newtheorem{corollary}{Corollary}[theorem]
\begin{corollary}
\label{localAreGlobal}
Only if a decision function $\dec$ satisfies the constraints required for greedy ascent to reach a global maximum from any point can the gradient-following policy be an optimal advice policy for $\dec$. This means all local maxima must also be global maxima for a decision function to have an optimial policy independent of resources.

\end{corollary}

\newtheorem{theorem2}{Theorem}
\begin{theorem}
\label{Linear}

For a continuous action space in $L_2$, a greedy policy is optimal if the gradient field of $D$ consists of straight lines wherever $\del \dec \ne 0$.

\end{theorem}
\newtheorem*{remark}{Remark}

\begin{remark}
    This explains why gradient-following is not optimal in Figure \ref{fig:2di}, in spite of the function being monotonic and otherwise amenable to gradient ascent.
\end{remark}

\newtheorem{corollary2}{Corollary}[theorem2]
\begin{corollary}
\label{thm:samedirection}
For a continuous action space in $L_2$, a greedy policy is optimal only if the gradient field of $\dec$ consists of straight lines wherever $\del \dec \ne 0$.
\end{corollary}
\begin{remark}
    By Corollary \ref{thm:samedirection}, the gradient-following policy is optimal for linear and logistic regression in $L_2$.
\end{remark}

%% file: 5_Experiments.tex
\begin{figure}
    \centering
\minipage{0.32\textwidth}
    \includegraphics[width=\linewidth]{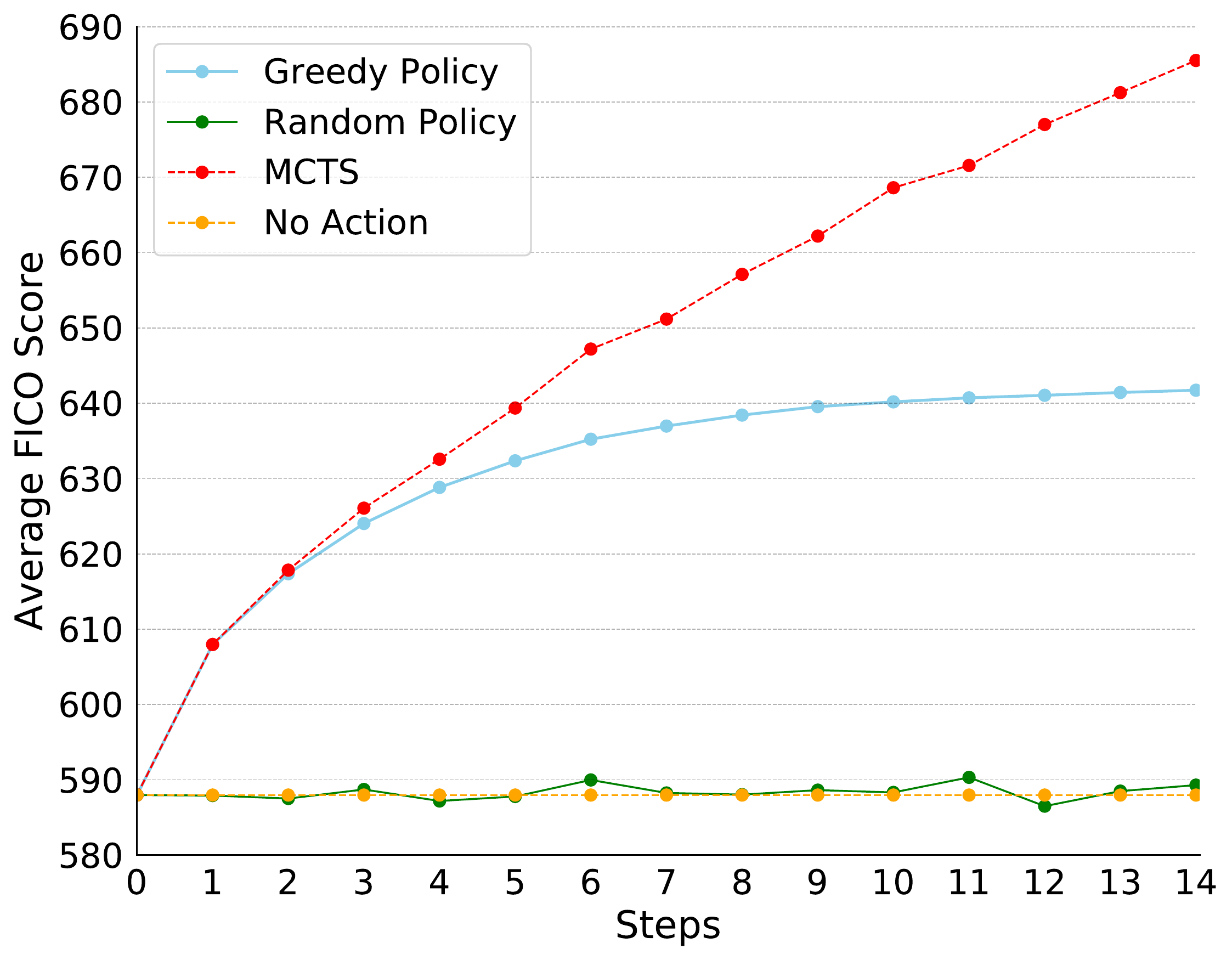}
\endminipage\hfill
\minipage{0.32\textwidth}
\includegraphics[width=\linewidth]{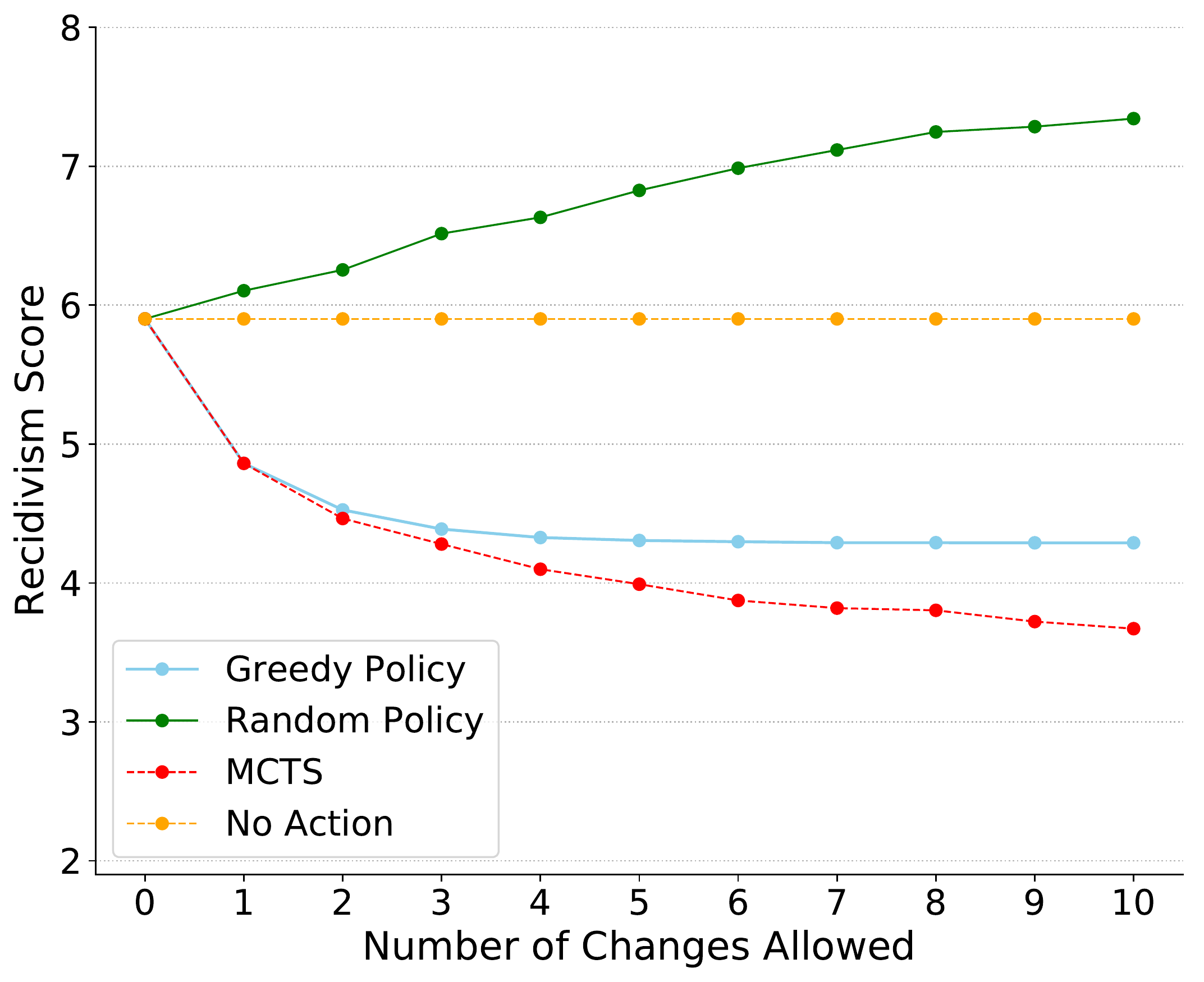}
\endminipage\hfill
\minipage{0.32\textwidth}
\includegraphics[width=\linewidth]{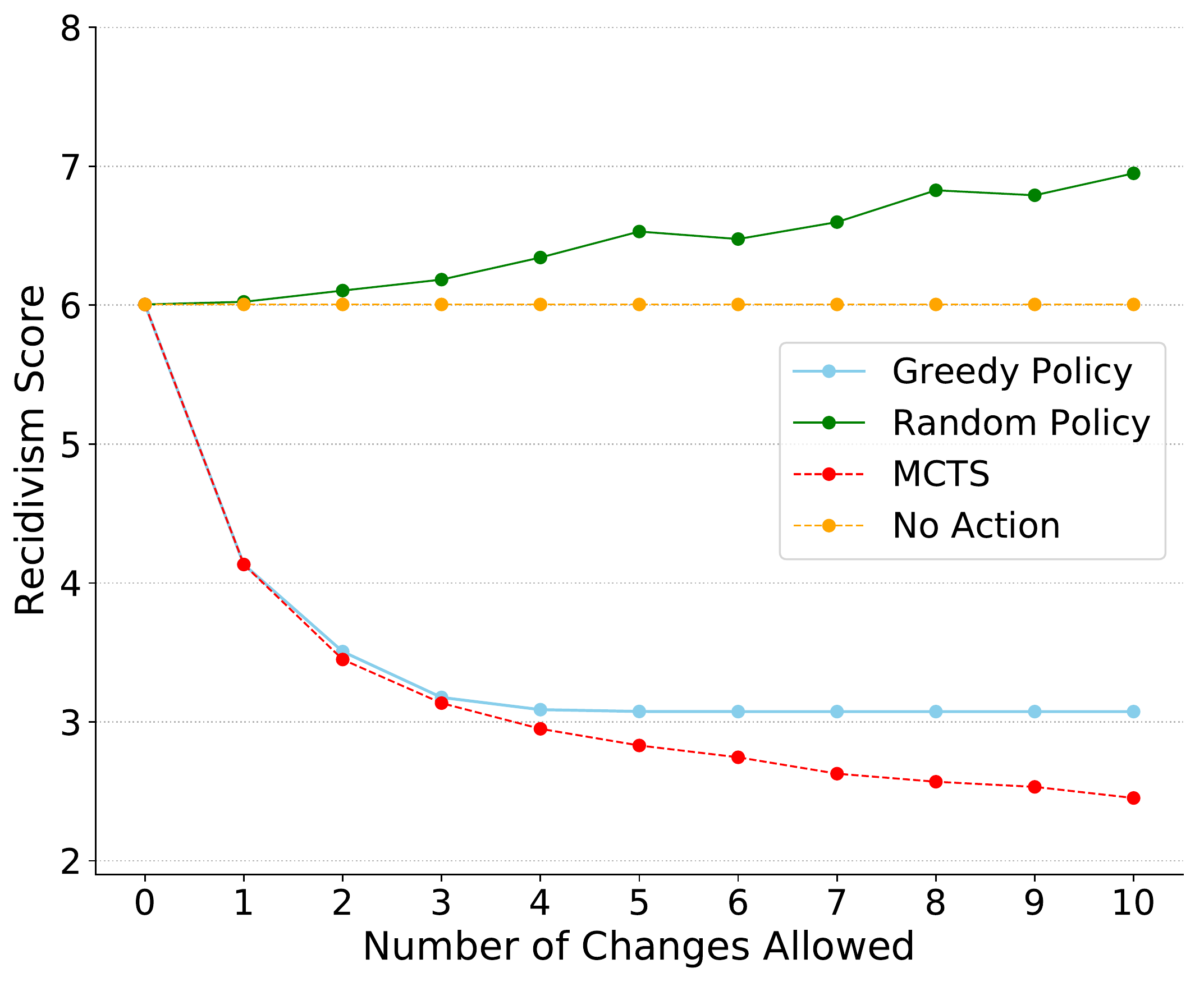}
\endminipage\hfill
    \caption{Comparing the performance of different advice policies, as defined in Eq. \ref{eq:incentivequality}, varying the initial resource count.
    \textbf{Left:} Simple Credit model (averaged over $1000$ initial states, higher is better). \textbf{Center:} recidivism prediction, including race and gender (averaged over $1000$ initially medium/high-risk states, lower is better). \textbf{Right:} Recivism prediction (excluding race/gender).}
    \label{fig:policycomparison}
\end{figure}
\section{Experiments}
\label{sec:experiments}
We applied our incentive-evaluation framework to two decision-settings: pretrial risk assessment, and credit scoring. In all cases, we used a discrete action space, and defined $\text{end}(s)$ as whether the agent has expended all their time/changes/resources. Note that 

\subheading{Incentive-Approximating Policies}

We compared different incentive-generating policies. To approximate the optimal agency-MDP policy, we used Monte-Carlo Tree Search (\textbf{MCTS}, \cite{browne2012survey}), with random rollouts and $0.5-1$s processing time. For the ``realistic'' credit model (with only 11 actions), we used an exhaustive search of all possible action sequences (\textbf{BFS}).
We also trained a double deep Q-network \cite{van2016deep} on the agency MDP, but found that in both settings the network generally failed to learn a meaningful advice policy (not even equaling the greedy policy), and so we have excluded those results.

We compare these incentives to the \textbf{greedy} policy (Eq. \ref{eq:greedy}), which maximizes the decision immediately after the current action, and to a \textbf{random} policy.

\subheading{Pretrial Risk Assessment Experiment}

In violent pretrial risk assessment, the criminal justice system outputs a score estimating the likelihood of an arrestee to be rearrested for a violent crime within two years, to determine whether to detain the individual pending trial. This offers an interesting case study of incentives, because in addition to accurately preventing violent crimes, the criminal justice system wants to incentivize decision-recipients to avoid criminal activity. We created a pretrial risk-assessment score using ProPublica's COMPAS dataset \cite{angwin_larson_kirchner_mattu_2019}, by training a random forest to use features about the arrestee and the charges from the current arrest to predict $2$-year violent recidivism, and then bucketing the subjects into score groups from 1 (lowest risk) to 10 (highest risk).
We trained two separate risk scores: one which was blind to the defendant's race and gender (immutable characteristics the subject could not affect), and one which had access to these characteristics. Both scoring algorithms could see a defendant's age - we found that excluding it made any resulting classifier highly inaccurate.

We defined the individual's state as their aforementioned decision features, plus a "number of remaining changes" they could make to their decision features. For a transition model, we allowed the individual, as a single action, to do one of: changing the type of crime (e.g. "drug", "violent", "theft"), incrementing or decrementing the degree of crime (e.g. "1st degree felony", "2nd degree misdemeanor"), or incrementing or decrementing the quantity of a certain type of previous interaction with the justice system (e.g. "\# of prior convictions").

As we can see in Figure \ref{fig:policycomparison}, the agency-MDP-solving policies matched or exceeded the greedy policy in providing agency to individuals across every setting.  Comparing the results between the two recidivism predictors, we see that, on average, excluding race and gender from the model increased individuals' ability to reduce their risk score. This is to be expected: as the decisions can no longer rely on these two immutable characteristics, they rely more heavily on attributes that individuals can hope to affect. 
In Figure \ref{fig:fairrecourse}, we further breakdown the impact of removing race and gender on individuals' agency. We see that the most substantial gains in agency occur among black men, with smaller gains among white men and black women. This provides an interesting counterpoint to conclusions from the fair machine learning literature \cite{dwork2012fairness}: while it may be true that algorithms blind to race or gender may still learn to discriminate based on correlates, blinding the algorithm to immutable characteristics may improve the agency of individuals from disadvantaged populations to change the decisions they receive. This presents an interesting direction for future work.

\begin{figure}
\minipage[t]{0.32\textwidth}
\includegraphics[width=\linewidth]{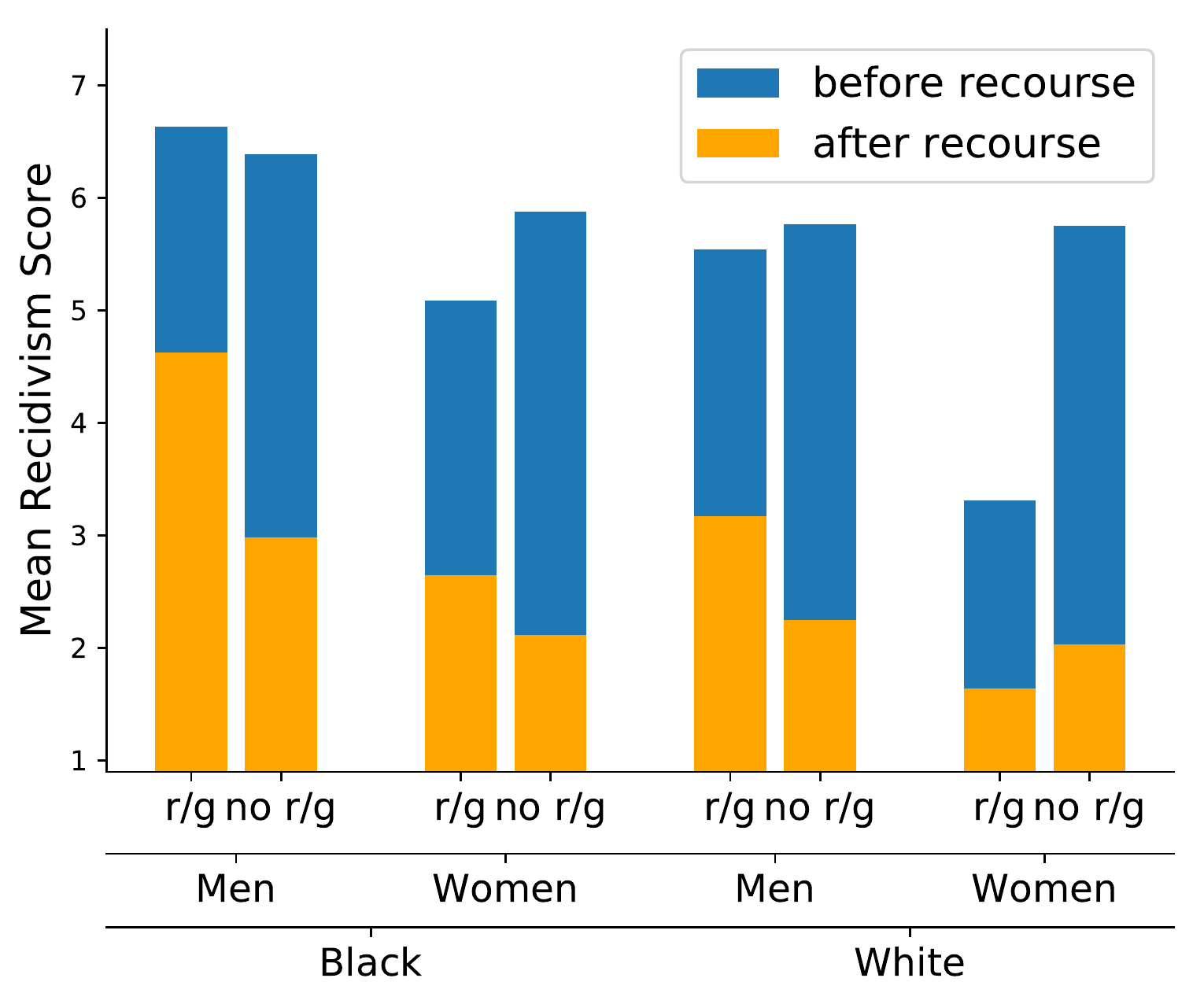}
\caption{Mean recidivism risk score 
before and after following MCTS-generated incentives for 6 to 10 steps, varying the inclusion of race/gender in decisions.}
    \label{fig:fairrecourse}
\endminipage\hfill
\minipage[t]{0.32\textwidth}
\includegraphics[width=\linewidth]{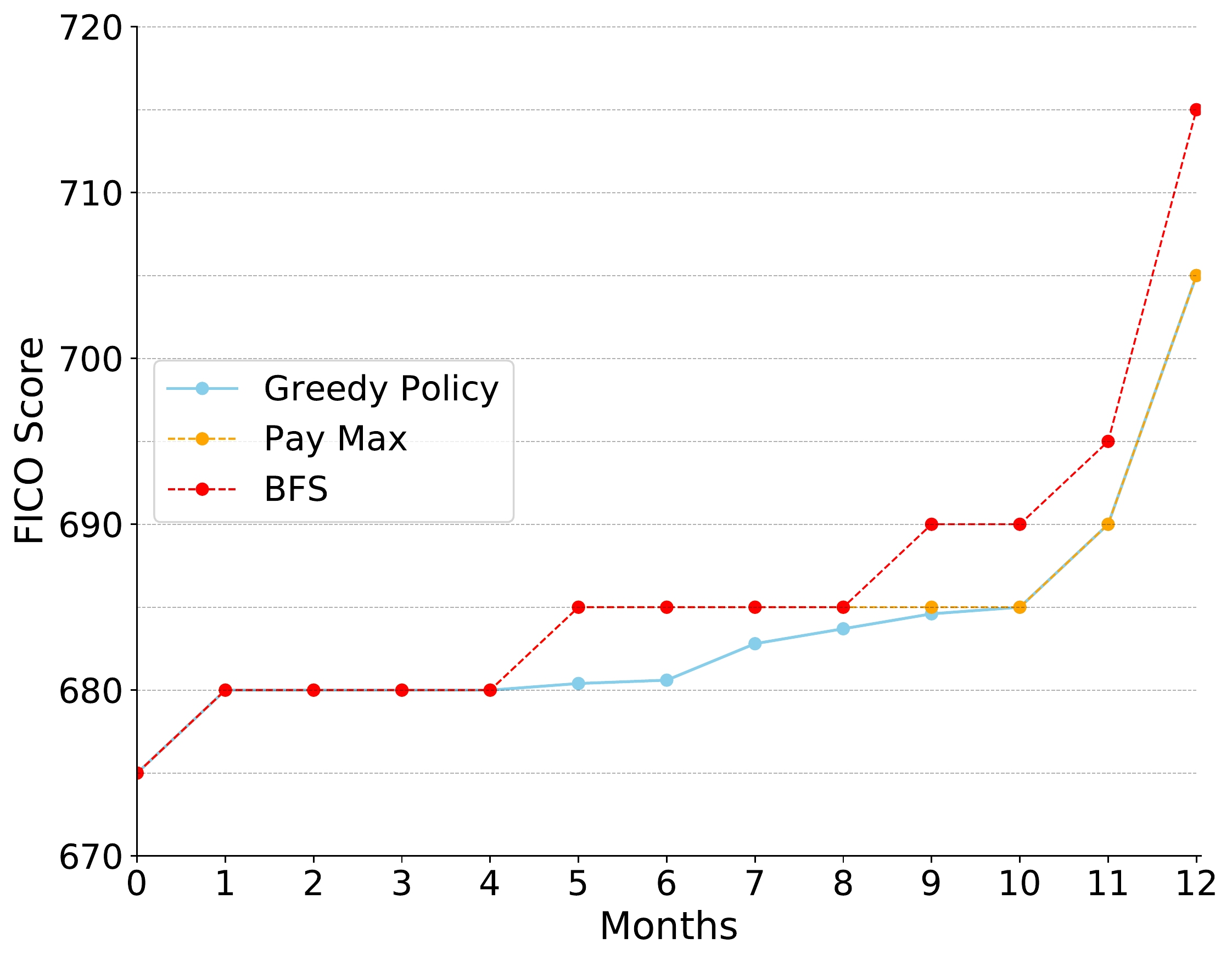}
\caption{Credit score under a realistic model, starting with US average financial data and no debt, and varying time remaining before score is checked.}
    \label{fig:averageamerican}
\endminipage\hfill
\minipage[t]{0.32\textwidth}
\includegraphics[width=\linewidth]{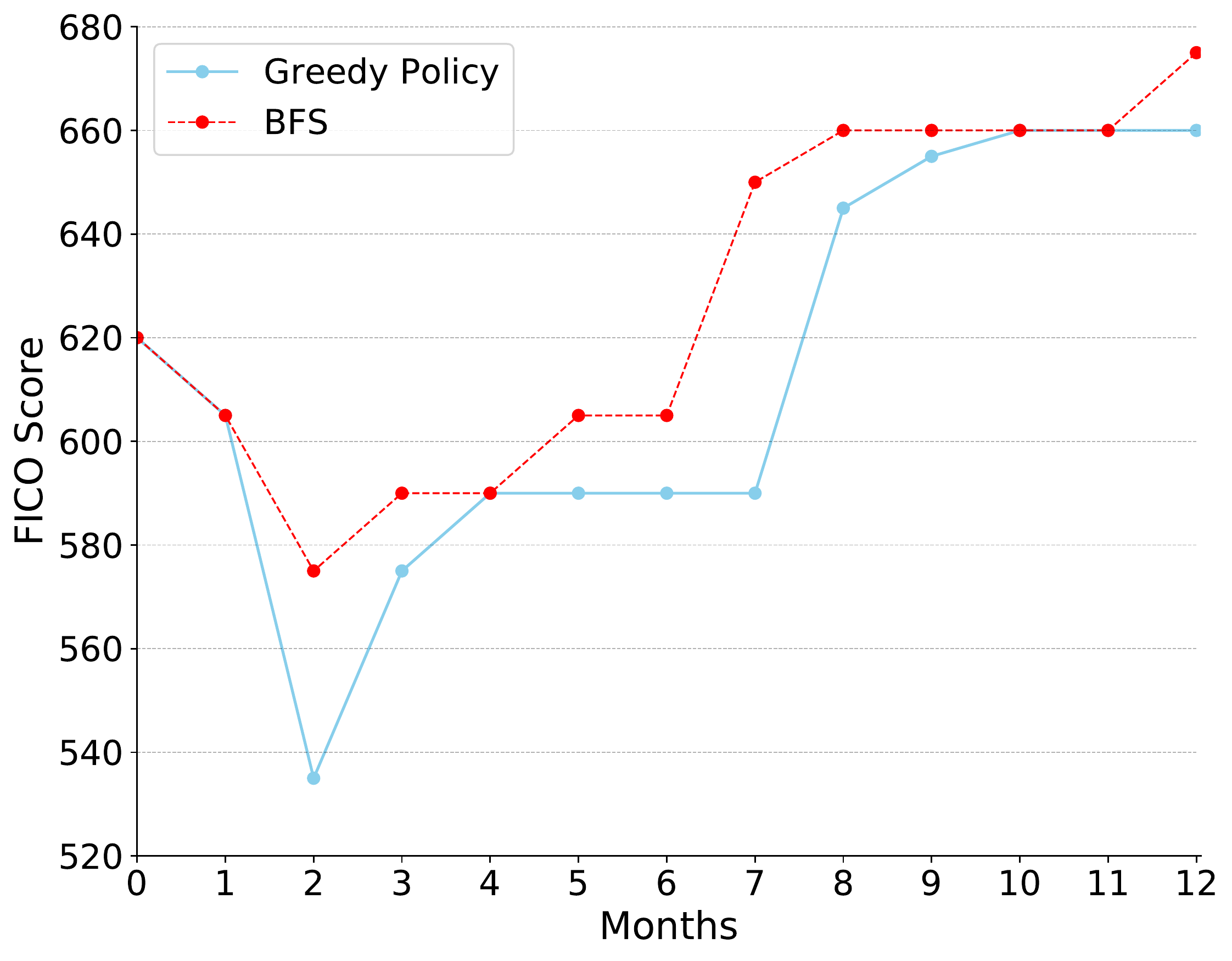}
\caption{Credit score under a realistic model, starting with US average financial data but a sudden crisis of \$10,000 of debt, and varying time remaining.}
    \label{fig:bankruptcy}
\endminipage\hfill
\end{figure}

\subheading{Credit Score Experiment}

For credit-scoring, we queried FICO's online credit-estimator \cite{myfico} at $~6\times10^6$ different points, and created a decision-function that output the same decision as the nearest queried point. FICO's estimator used answers to $10$ multiple-answer questions to asses a person's credit. Our "simple" model used these $10$ questions (plus the number of remaining actions) as the state, and allowed any question to be incremented or decremented as an action. We also created a ``realistic'' model (see Figure \ref{fico}), in which the state was comprised of grounded features such as `current cash on hand', `current debt',`how long each card has been open', `last missed payment', and `has declared bankruptcy'. The actions chosen for the transition model were meant to be concrete and achievable (e.g. ``open a credit card and pay the minimum on your cards this month'' rather than ``have less debt''). These actions had sequence-dependent effects like incurring monthly interest on unpaid loans, or declaring bankruptcy (an irreversible action which affected the credit score but simultaneously removed debt).

\begin{figure}
    \centering
\vspace*{-.5cm}
\includegraphics[width=0.5\textwidth, trim=0 60 0 0, clip]
{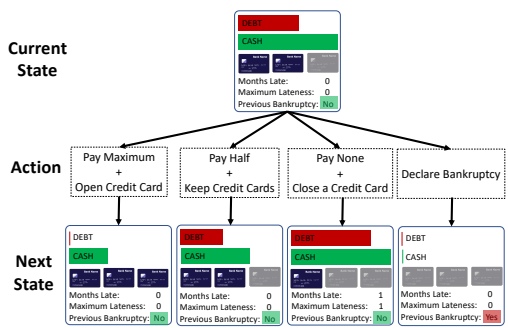}
\caption{Examples of actions an agent can take each month within the ``complex FICO'' model.}
\label{fico}
\vspace*{-.5cm}
\end{figure}

In Figures \ref{fig:averageamerican} and \ref{fig:bankruptcy}, we compared the performance of different advice policies for an average U.S. household (based on income, debt, credit limit, and interest rate) to try to increase their score \cite{frankel_2018,mccann_2019,josephson_josephson_2018}. As we describe below, the greedy policy consistently made short-sighted decisions, while the MDP-solving policy accepted short-term losses for larger long-term gains.

In Figure \ref{fig:averageamerican}, the individual starts off debt-free. The greedy policy always pays the full amount, occasionally opening up a single card in one of the first few months. However, as opening up a second card would decrease the short-term score, the greedy policy only opens a single card. For shorter time-spans, BFS acts similarly to the greedy policy, avoiding opening too many cards to prevent the decrease in credit. 
Given a longer period to act, however, BFS begins by opening up several cards (temporarily decreasing the score), which results in a higher overall score later on. We compared both policies to a ``pay max'' policy that always pays the full amount and doesn't open or close cards. This policy beats greedy on average over the timeframe as greedy doesn't open enough cards to fully reap the benefits of the increased credit limit, and sometimes suffers from the penalty of opening a card.

In Figure \ref{fig:bankruptcy}, the individual starts off with a sudden debt. The greedy policy starts by missing a payment (since the debt doesn't immediately impact the credit score by much), but the interest accrued from this large debt forces the agent to declare bankruptcy in month 2. Given only a month, BFS agrees with the greedy policy and simply misses a payment. However, given more than one month, BFS starts by declaring bankruptcy (knowing it'll happen eventually) so that it can begin rebuilding its credit as quickly as possible.

To highlight the importance of the choice of action-model, we also evaluated what the ``simple'' FICO model, with its simplistic actions, would recommend doing in these "realistic" credit scenarios. For the debt-free scenario, the "simple" BFS advice policy recommended "opening more cards" while keeping fixed the number of recent credit inquiries, and the date the most recent card was opened. While doing so would in fact improve one's credit score, this isn't useful advice, since opening a card without these additional consequences is impossible. For the debt-ridden scenario, the ``simple'' advice policy recommended having less credit card debt -- which, again, is correct but useless advice, since simply "having less debt" is not an executable action.


%% file: 6_Discussion.tex
\section{Discussion}
\label{sec:discussion}

In this work, we have proposed a framework for reasoning about the incentives of individuals subject to black box decisions, by formulating their choices as an MDP. We demonstrated the framework's ability to analyze the agency of different groups subject to a recidivism predictor, and illustrated the framework's ability to extract superior advice policies for a customer subject to FICO's credit simulator.

Interestingly, this framework can be used for two different ends: comparing the utility of different advice-generating algorithms (e.g. local approximation, MDP-based tree search) on a fixed decision model, and comparing the incentives generated by different decision models while keeping the advice-generating algorithm fixed. We suggest that it may be most useful to choose an advice-generating scheme first (either based on MDP performance guarantees, or qualities like human-interpretability). Once a scheme has been chosen, we can reasonably compare the incentives/agency that different decision models provide to agents who we assume will follow that same advice scheme. 

Throughout this work we have assumed that a decision-making model does not change based on an individual's actions. However, when many individuals follow incentives at once, this is liable to shift the data distribution, and thus the model. The literature on strategic classification responds by attempting to remove agency from the individual. One interesting open question is whether we can incorporate the impact of subjects' actions on the decision-maker in such a way that the decision-maker will continue to provide agency to those subjects.


One primary ethical danger looms over the study of algorithmic incentives. What if focusing on how an individual could improve the decision they receive from an existing model will distract us from focusing on larger questions, namely: is the basis of the decision correct? And more importantly, does the decision-maker have a right to dictate our actions?

These are vital questions for practitioners to consider. We take solace in a simple truth:
a first step in the process of eliminating damaging incentives is having the tools to surface those incentives, so that they can be confronted and expunged.

%% file: 7_Ack.tex
\section{Acknowledgements}

The authors would like to thank Christina Ilvento for her early insights into the challenges of identifying incentives, and to Cynthia Dwork, Michael P. Kim, Suhas Vijaykumar, Boriana Gjura, and Charles Leiserson for their helpful feedback and advice.

William S. Moses was supported in part by a DOE Computational Sciences Graduate Fellowship DE-SC0019323. This research was supported in part by NSF Grant 1533644 and 1533644, in part by LANL grant 531711, and in part by IBM grant W1771646.

%% file: A_Proofs.tex
\section{Proofs}
\label{sec:proofs}

\newtheorem{atheorem}{Theorem}
\newtheorem{acorollary}{Corollary}[atheorem]
\setcounter{atheorem}{0}

As mentioned in Section \ref{sec:linear}, we define a locally-optimal \textit{greedy} policy as a policy that chooses actions based only on maximizing the immediate improvement in the received decision:
\begin{equation*}
\pi_{local}(s) = \argmax_{a \in \acspace} \E_{s' \sim \tran(s,a)} \left [\dec\left(I_x(s')\right)\right ]
\end{equation*}
When the function is differentiable, we can also define the \textit{gradient-following} policy for a state $s = [x, r]$ composed of only the decision features $x$ and the remaining available displacement $r$. Note that in this hypothetical transition model, we can manipulate every feature axis in $x$ independently.
\begin{equation*}
\pi_{gradient}(s) = \argmax_{a \in\{a' \in \RR^{dim(x)}~| ~\|a'\|_2 = \epsilon, ~\epsilon \ll 1\} } \dec(x + a) ~~= ~~ \epsilon \nabla_x\dec\left(x\right)
\end{equation*}

\begin{atheorem}
If there exists a policy that is optimal independent of the number of resources remaining, it must be equivalent to the greedy policy.
\end{atheorem}
\begin{proof}
Suppose our policy $\pi(s)$ does not consider the number of resources left. If $\pi(s)$ is optimal, it must be optimal given the minimum number of resources for there to exist a valid action (in the case of a continuous number of resources we may use the limit of $\text{resources} \rightarrow 0$).
\end{proof}

\begin{acorollary}
Only if a decision function $\dec$ satisfies the constraints required for greedy ascent to reach a global maximum from any point can the greedy policy be an optimal advice policy for $\dec$. This means all local maxima must also be global maxima for a decision function to have an optimial policy independent of resources.

\end{acorollary}

\begin{atheorem}
For a continuous action space in $L_2$, a greedy policy is optimal only if the gradient field of $\dec$ consists of straight lines wherever $\del \dec \ne 0$.
\end{atheorem}
\begin{proof}
For a continuous action space in $L_2$, a greedy policy is equivalent to the gradient-following policy.

Consider a point $x_0$ where $\del \dec(x_0) \ne 0$. Thus $\dec(x_0)$ is not a global maximum since we can follow $\del \dec(x_0)$ to a more optimal point. 

Suppose $\dec$ has a global maximum.  Let $x^*$ be the closest point to $x_0$ where $\dec(x^*)$ is a global maximum. The optimal policy for $|x^* - x_0|$ resources starting at $x_0$ is a straight line from $x_0$ to $x^*$, since any other policy would not reach a global maximum. 

Suppose $\dec$ has no global maximum. By Corollary \ref{localAreGlobal}, $\dec$ has no local maxima. Let $r$ be an arbitrary resource quantity. Let $x^* = \argmax_{x, |x-x_0| \le r} \dec(x)$, the best point reachable from $x$ with $s$ resources. Since $\dec$ has no local maxima, $x^*$ must be on the boundary of all reachable destinations from $x_0$ (assuming $\dec$ is not constant). Therefore the optimal policy for $r$ resources starting at $x_0$ is a straight line from $x_0$ to $x^*$.

Since the optimal policy follows a straight line wherever $\del \dec \ne 0$, if gradient-following is an optimal policy, then the gradient field of $\dec$ must be straight lines where $\del \dec \ne 0$.
\end{proof}

\begin{acorollary}
For a continuous action space in L2, if $\del \dec \ne 0$ everywhere, then the gradient-following policy is optimal if and only if all gradients point in the same direction $\vec c$ and have equal magnitude in the plane tangent to $\vec c$.
\end{acorollary}

\begin{proof}
Since $\del \dec \ne 0$, Theorem \ref{Linear} states that an agent following the gradient starting at $x_0$ will draw a ray to infinity. This ray can be extended backwards to form a line, such that following the gradient starting at any point on this line will keep the agent on the line, moving in one direction. This is also true for an arbitrary second point $x_1$. These two lines must be parallel.

If these two lines were not parallel, they must eventually meet at a point $y$. Therefore at $y$ the gradient must point in two directions at once. Since $\del \dec \ne 0$ this is not possible, generating a contradiction.

Therefore all gradient lines of $\dec$ must be parallel, pointing in some direction $\vec c$.

Moreover, consider the plane tangent to $\vec c$ at an arbitrary $x$. Gradients at all points on this plane must not only point in the same direction, but have the same magnitude.  A full proof is omitted for brevity, but intuitively if this were not the case, it would be optimal for an agent at a point with smaller gradient to move slightly in the direction of a point in the same plane with a larger gradient, creating a contradiction. Additionally, the value of the decision function must be equal at all points in this plane (otherwise the gradient would not be tangent to the plane).

Last, we prove gradient-following is optimal. 
As the decision function is equal at all values in planes tangent to $\vec c$, it sufficies to show gradient-following picks the maximum value along $\vec c$. Since the gradient points in the direction of $\vec c$ and is nonzero, gradient following is optimal.%
%
\end{proof}